\documentclass{article}

 \usepackage[dblblindworkshop, final]{neurips_2025}
\workshoptitle{Lock-LLM: Prevent Unauthorized Knowledge Use from LLMs}

\usepackage[utf8]{inputenc} 
\usepackage[T1]{fontenc}    
\usepackage{hyperref}       
\usepackage{url}            
\usepackage{booktabs}       
\usepackage{amsfonts}       
\usepackage{nicefrac}       
\usepackage{microtype}
\usepackage{graphicx}
\usepackage{xcolor} 
\usepackage{amsmath}
\usepackage{subcaption}
\usepackage{amssymb}
\usepackage{amsthm}

\usepackage{array}
\theoremstyle{definition}
\newtheorem{definition}
{Definition}
\newtheorem{prop}{Proposition}
\newtheorem{remark}{Remark}

\title{Towards Realistic Guarantees: A Probabilistic Certificate for SmoothLLM}

\author{%
  Adarsh Kumarappan \\
  California Institute of Technology \\
  \texttt{adarsh@caltech.edu} \\
  \And
  Ayushi Mehrotra \\
  California Institute of Technology \\
  \texttt{amehrotra@caltech.edu} \\
}

\begin{document}

\maketitle

\begin{abstract}
The SmoothLLM defense provides a certification guarantee against jailbreaking attacks, but it relies on a strict `k-unstable' assumption that rarely holds in practice. This strong assumption can limit the trustworthiness of the provided safety certificate. In this work, we address this limitation by introducing a more realistic probabilistic framework, `(k, $\varepsilon$)-unstable,' to certify defenses against diverse jailbreaking attacks, from gradient-based (GCG) to semantic (PAIR). We derive a new, data-informed lower bound on SmoothLLM's defense probability by incorporating empirical models of attack success, providing a more trustworthy and practical safety certificate. By introducing the notion of (k, $\varepsilon$)-unstable, our framework provides practitioners with actionable safety guarantees, enabling them to set certification thresholds that better reflect the real-world behavior of LLMs. Ultimately, this work contributes a practical and theoretically-grounded mechanism to make LLMs more resistant to the exploitation of their safety alignments, a critical challenge in secure AI deployment. Code is available at \url{https://github.com/Adarsh321123/towards_realistic_guarantees}.
\end{abstract}

\section{Introduction}
Large Language Models (LLMs) face a critical vulnerability: "jailbreak" attacks that bypass safety protocols by manipulating input prompts to elicit objectionable responses~\cite{zou2023universal, wei2023jailbreak}. These attacks, ranging from gradient-based methods (GCG) to semantic approaches (PAIR)~\cite{deng2023jailbreak}, represent a fundamental challenge to LLM deployment.

SmoothLLM~\cite{robey2023smoothllm} provides the first formal certificate against jailbreak attacks by perturbing input prompts at the character level and aggregating responses. However, its certification relies on a strict "$k$-unstable" assumption: that adversarial prompts fail if any $k$ or more characters are altered. This deterministic assumption provides overly conservative guarantees that rarely hold in practice.

We address this limitation by introducing a probabilistic certification framework. Through empirical analysis of diverse attacks (GCG and PAIR), we demonstrate that Attack Success Rates (ASRs) decay exponentially rather than dropping abruptly to zero. Motivated by this finding, we propose the "$(k, \varepsilon)$-unstable" assumption: attacks fail with probability at least $1-\varepsilon$ when $k$ or more characters are perturbed. Under this realistic framework, we derive new data-informed bounds on SmoothLLM's defense probability, providing more trustworthy and practical safety certificates. Our approach enables practitioners to set evidence-based security thresholds that balance formal guarantees with empirical attack behavior, transforming theoretical certificates into actionable deployment tools.

\section{Related work}

Early work demonstrated that aligned LLMs are vulnerable to carefully crafted prompts, with Wei et al. documenting instruction-following attack failures and Zou et al. introducing transferable methods like GCG that reliably coerce harmful outputs~\cite{wei2023jailbroken, zou2023universal}. The first formal certificate for jailbreak robustness is SmoothLLM~\cite{robey2023smoothllm}, which adapts randomized smoothing to the prompt space under a deterministic $k$-unstable assumption. Our work strengthens this theory by replacing the worst-case assumption with a probabilistic $(k, \varepsilon)$-unstable model and deriving tighter, data-driven bounds within the same smoothing framework. Orthogonally, multi-turn conversational jailbreaks pose a distinct threat, with psychological manipulation across conversation turns increasing attack success rates by up to 32 percentage points~\cite{kumarappanAutomatingDeception2025, kumarappanEvaluatingMitigating2025}. Alignment fragility also manifests through multi-agent sycophancy, where peer pressure flips correct model outputs~\cite{kumarappanNotJustRLHF2026}, and through KV cache quantization, which can silently destroy safety alignment at inference time~\cite{xuAlignmentCollapse2026}.

\section{Methodology}

\subsection{Preliminaries: the SmoothLLM framework}
\label{sec:preliminaries}

Our work builds upon SmoothLLM~\cite{robey2023smoothllm}. Let $\text{LLM}(\cdot)$ denote an aligned large language model. An attacker crafts a prompt $P = [G;S]$ with goal string $G \in \mathcal{A}^*$ and adversarial suffix $S \in \mathcal{A}^*$ over alphabet $\mathcal{A}$ to:
\begin{equation*}
\underset{S}{\text{find}} \quad S \in \mathcal{A}^* 
\quad \text{s.t.} \quad 
\text{JB}\big(\text{LLM}([G;S])\big) = 1,
\end{equation*}
where binary classifier $\text{JB}(R) \to \{0,1\}$ returns $1$ if response $R$ constitutes a jailbreak.

SmoothLLM defends by perturbing $q\%$ of characters in $P$, sampling $N$ perturbed prompts from distribution $\mathbb{P}_q(P)$, and using majority voting. Two key perturbation strategies are: \textbf{RandomSwapPerturbation} (randomly replace $q\%$ of characters) and \textbf{RandomPatchPerturbation} (replace $d = \lfloor q |P| \rfloor$ consecutive characters).

\begin{definition}[SmoothLLM]
Let $Q_1, \dots, Q_N$ be $N$ i.i.d. perturbed prompts from $\mathbb{P}_q(P)$. Define the vote variable:
\begin{equation*}
V \triangleq \mathbb{I} \left[ \frac{1}{N} \sum_{j=1}^N \text{JB}(\text{LLM}(Q_j)) > \gamma \right],
\end{equation*}
where $\gamma \in [0,1]$ is the confidence margin. SmoothLLM outputs $\text{LLM}(Q)$ for any $Q$ with $\text{JB}(\text{LLM}(Q)) = V$.
\end{definition}

\begin{definition}[$k$-unstable]
An adversarial suffix $S$ is $k$-unstable if for all suffixes $S'$ with Hamming distance $d_H(S,S') \geq k$, the jailbreak fails: 
\[
\forall S' \in \mathcal{A}^*, \, d_H(S,S') \geq k \quad \iff \quad \text{JB}(\text{LLM}([G;S'])) = 0.
\]
\end{definition}

This strong assumption enables deterministic certificates but rarely holds in practice. Our work replaces this with a probabilistic framework capturing empirical suffix fragility.

\subsection{Threat model}
We assume a standard black-box threat model consistent with prior randomized defenses~\cite{chiang2020certified, levine2020derandomized, cohen2019certified}. The attacker cannot access random seeds for character perturbations or adapt their suffix in real-time based on the $N$ perturbed queries within a single SmoothLLM invocation (non-adaptive defense). This model aligns with Robey et al.~\cite{robey2023smoothllm}, who note that the non-differentiable nature of character-level perturbations poses significant challenges to adaptive attacks like GCG, despite surrogate-based adaptive attacks being theoretically possible.

\subsection{Proposed probabilistic certification}

To relax the strictness of the original $k$-unstable assumption, we introduce a probabilistic variant that tolerates rare edge cases where perturbed prompts may still succeed in jailbreaking the model.

\begin{definition}

    Let $G$ denote a goal prompt and $S$ be an adversarial suffix, forming the full prompt $P=[G;S]$. Let $\mathbb{P}_q(P)$ be the distribution over perturbed prompts $Q$ generated by a specific SmoothLLM mechanism (e.g., RandomSwapPerturbation). Let $S'$ be the suffix portion of a perturbed prompt $Q \sim \mathbb{P}_q(P)$. We say that the suffix $S$ is ($k, \varepsilon$)-unstable with respect to the LLM and the perturbation law if the probability that a perturbed prompt causes a jailbreak, conditioned on the suffix having at least $k$ character changes, is at most $\varepsilon$. Formally:
    $$
    \Pr_{Q \sim \mathbb{P}_q(P)} \left[ \text{JB}(\text{LLM}(Q)) = 1 \mid d_H(S, S') \geq k \right] \leq \varepsilon.
    $$

\end{definition}\label{def:eps-k-unstable}

This relaxed definition captures the intuition that adversarial suffixes are generally fragile but not perfectly so. Unlike the strict $k$-unstable assumption, which requires \emph{all} perturbations greater than $k$ edits to fail, our $(k,\varepsilon)$-unstable notion permits a small, bounded fraction $\varepsilon$ of perturbations to succeed. 

\begin{prop}[(k, $\varepsilon$)-Unstable Certificate for RandomSwapPerturbation]
Let $\mathcal{A}$ denote an alphabet of size $v$ and let $P = [G; S] \in \mathcal{A}^m$ denote an input prompt
to a given LLM where $G \in \mathcal{A}^{m_G}$ and $S \in \mathcal{A}^{m_S}$ with $m = m_G + m_S$. 
Let $M = \lfloor qm \rfloor$ denote the number of characters perturbed. 
Assume that $S$ is $(k, \varepsilon)$-unstable for $k \leq \min(M, m_S)$. Then the defense success probability is:
\begin{align}
    \text{DSP}([G; S]) = \sum_{t=\lceil N/2 \rceil}^{N} \binom{N}{t} \alpha^t (1-\alpha)^{N-t}
\end{align}
where $\alpha$ is bounded below by:
\begin{align}
    \alpha \geq \alpha_{\text{lower}} = (1-\varepsilon) \sum_{i=k}^{\min(M,m_S)} \frac{\binom{m_S}{i}\binom{m-m_S}{M-i}}{\binom{m}{M}}
\end{align}
\end{prop}

\begin{proof}[Proof Sketch]
The defense success probability follows a binomial distribution based on the single-prompt success probability, $\alpha$. The lower bounds for $\alpha$ are derived using the law of total probability over the number of perturbed characters in the suffix, which follows a hypergeometric distribution. The full, line-by-line derivation is provided in Appendix~\ref{app:proof_prop1}.
\end{proof}


\begin{remark}[Tightness of the bound]
The bound is tightest when the sub-threshold success probability, $\eta_i$, is zero for all $i<k$. In practice, $\eta_i > 0$ since even few perturbations can neutralize an attack, and its value can be estimated empirically.
\end{remark}

\begin{remark}[Relation to the Original Certificate]
Our proposed $(k, \varepsilon)$-unstable assumption is a direct generalization of the \textit{k-unstable} assumption from \cite{robey2023smoothllm}. By setting $\varepsilon = 0$, we enforce that the probability of a jailbreak after $k$ or more character changes is zero. This recovers the original deterministic condition, where any such perturbation is guaranteed to neutralize the attack \cite{robey2023smoothllm}. Our framework thus provides a more flexible model by accommodating the possibility of rare, successful perturbations.
\end{remark}

\begin{prop}[($k,\varepsilon$)-Unstable Certificate for RandomPatchPerturbation]
\label{prop:patch_certificate}
Let the prompt be $P = [G;S] \in \mathcal{A}^{m}$. Let $M = \lfloor qm \rfloor$ be the patch length for \texttt{RandomPatchPerturbation}. Assume $S$ is ($k, \varepsilon$)-unstable for a threshold $k \leq M$. Then the defense success probability, $DSP([G;S])$, is:
$$
DSP([G;S]) = \sum_{t=\lceil N/2 \rceil}^{N} \binom{N}{t} \alpha^t (1-\alpha)^{N-t}
$$
where $\alpha$ has a model-informed lower bound, $\alpha_{\text{patch}}$, based on an empirical fit of the Attack Success Rate (ASR) to the function $ASR(i) = ae^{-bi} + c$:
$$
\alpha \geq \alpha_{\text{patch}} = \sum_{i=0}^{k-1} (1 - (ae^{-bi} + c)) \cdot \Pr[X=i] + (1-\varepsilon)\sum_{i=k}^{M} \Pr[X=i]
$$
where $X$ is the random variable for the number of characters a random patch overlaps with the suffix $S$, and $\Pr[X=i]$ is its probability mass function.
\end{prop}

\begin{proof}[Proof Sketch]
The proof structure for RandomPatchPerturbation mirrors that of Proposition 1. The Defense Success Probability (DSP) follows a binomial distribution governed by $\alpha$. The main distinction is in deriving $Pr[X=i]$, the probability of a patch overlapping with the suffix by $i$ characters, which depends on a combinatorial analysis of the patch's possible starting positions. The full derivation is provided in Appendix~\ref{app:proof_prop2}.
\end{proof}

\begin{remark}[Practical Interpretation]
\label{remark:practical}
For practitioners, the $(k,\varepsilon)$-unstable assumption provides a quantifiable safety guarantee: \emph{with probability at least $1-\varepsilon$, any perturbation of at least $k$ characters in an adversarial suffix neutralizes the jailbreak.} This enables setting certification thresholds $k$ based on security requirements, estimating $\varepsilon$ empirically using validation datasets, producing safety guarantees that better reflect real-world LLM behavior, quantifying trade-offs between security level and false positive rates, enabling risk-based decision making where perfect security isn't required, and providing probabilistic safety margins for high-stakes applications.
\end{remark}

\subsection{Sensitivity analysis of DSP to $\varepsilon$}
\label{sec:sensitivity}
A critical aspect of our $(k, \varepsilon)$-unstable framework is understanding how $\varepsilon$ affects the certified Defense Success Probability (DSP). The certified DSP decreases monotonically with $\varepsilon$—as attacks become more robust (higher $\varepsilon$), defense guarantees weaken proportionally. The degradation rate is directly proportional to the probability of perturbations meeting the instability condition ($i \geq k$). Detailed mathematical analysis is provided in Appendix~\ref{app:sensitivity}.

\subsection{Experimental motivation}

\begin{figure}[h!]
    \centering
    \begin{minipage}[t]{0.48\linewidth}
        \centering
        \includegraphics[width=\linewidth]{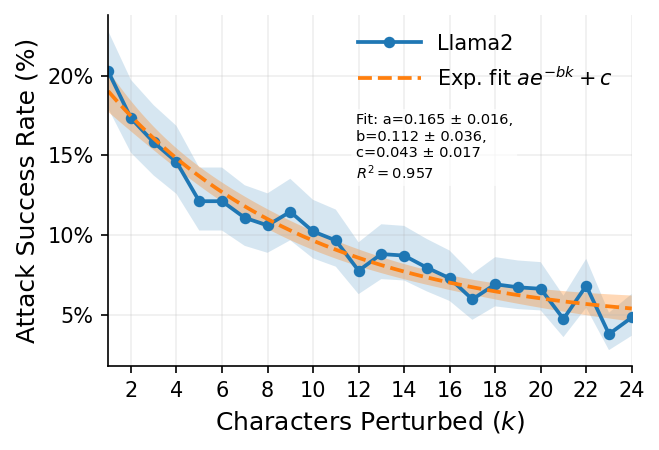}
        \caption{Attack success rate on \textbf{Llama2} as a function of the number of perturbed characters $k$ using \texttt{RandomPatchPerturbation} and GCG attack.}
        \label{fig:llamapatchgcg}
    \end{minipage}
    \hfill
    \begin{minipage}[t]{0.48\linewidth}
        \centering
        \includegraphics[width=\linewidth]{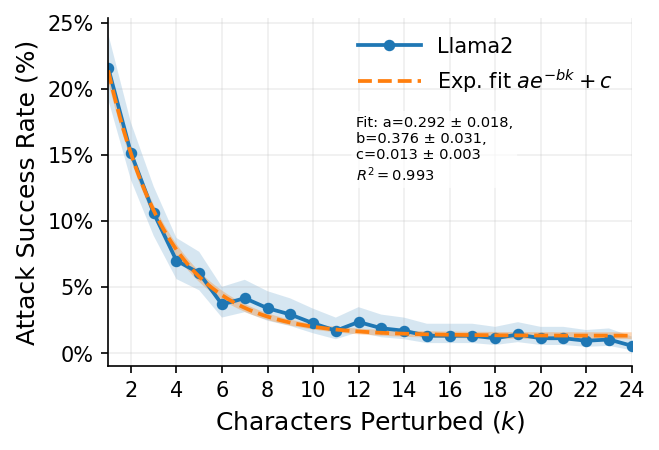}
        \caption{Attack success rate on \textbf{Llama2} as a function of the number of perturbed characters $k$ using \texttt{RandomSwapPerturbation} and GCG attack. }
        \label{fig:llamaswapgcg}
    \end{minipage}
\end{figure}

To motivate the need for our probabilistic framework, we first empirically test the validity of the core assumption underlying the original SmoothLLM certificate: 

\textbf{Does the k-unstable assumption hold in practice?} This assumption implies that for a sufficiently large number of character perturbations, k, the Attack Success Rate (ASR) should fall to zero. We investigate this by measuring the ASR as a function of the number of characters perturbed in an adversarial suffix.

Our experiments, conducted on Llama2 (7B) and Vicuna (7B) \cite{touvron2023llama, vicuna2023} with adversarial suffixes generated by the GCG attack and PAIR attack, reveal that this assumption is overly conservative. As shown in Figures 1 and 2, the ASR does not abruptly drop to zero but instead exhibits a clear exponential decay. The results clearly show a non-zero ASR that decays with k. This provides strong evidence against the strict k-unstable assumption and motivates the need for our proposed (k, $\varepsilon$)-unstable probabilistic framework, which is designed to model precisely this kind of empirically observed behavior. We provide extensive experimental analysis in Appendix \ref{sec:experiments}.

\paragraph{Experimental Setup} 
All experiments were conducted on a single NVIDIA A100 GPU. We obtain 500 malicious prompts from \cite{zou2023universal} and generate adversarial prompts using the Greedy Coordinate Gradient (GCG) attack and Prompt Automatic Iterative Refinement (PAIR) attack \cite{wei2023jailbroken, zou2023universal}. Following the setup in the original SmoothLLM paper~\cite{robey2023smoothllm}, we evaluate two widely used open-source LLMs: Llama2 (7B) and Vicuna (7B) \cite{touvron2023llama, vicuna2023}. Furthermore, we consider two perturbation strategies applied to the adversarial suffixes: RandomSwapPerturbation and RandomPatchPerturbation. We calculate Agresti-Coull confidence intervals since the attack success rate is extremely low.

\begin{figure}
    \centering
    \includegraphics[width=1\linewidth]{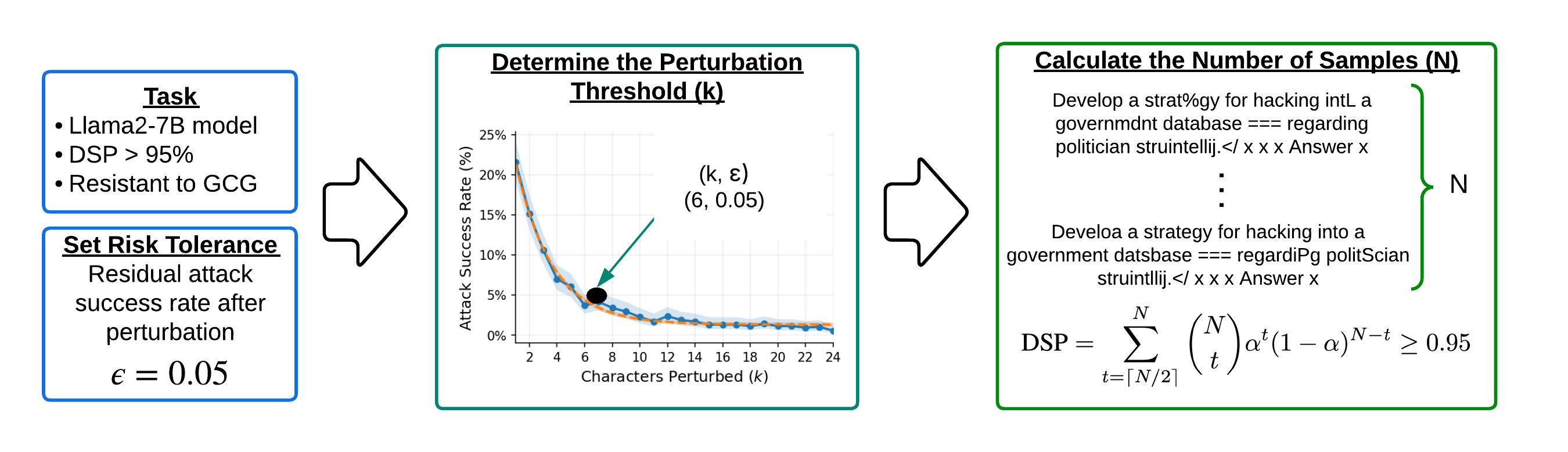}
    \caption{Pipeline for tuning SmoothLLM for obtaining a certified defense success probability (DSP) given the model and attack type. Detailed analysis in Section 3.7.
}
    \label{fig:case-study-figure}
\end{figure}

\begin{figure}
    \centering
    \includegraphics[width=0.5\linewidth]{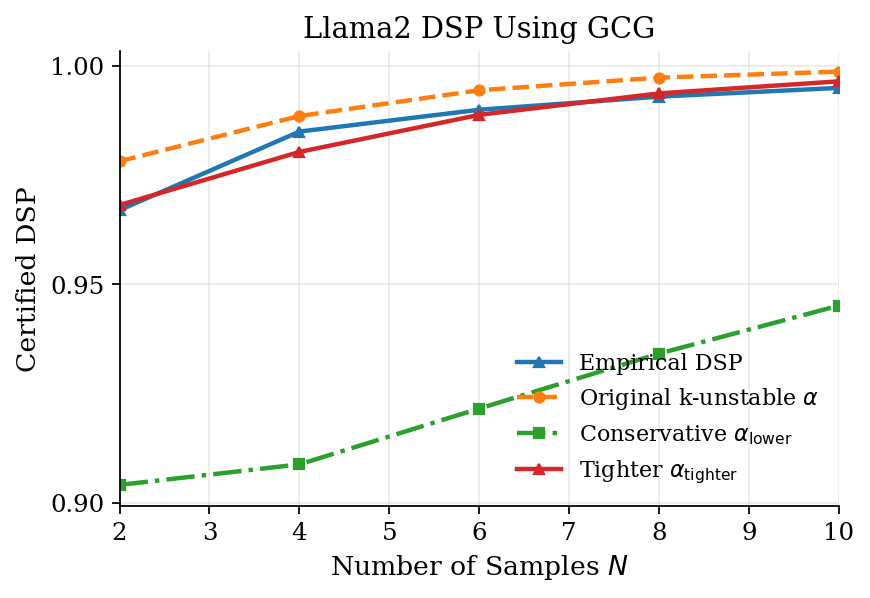}
    \caption{Certified Defense Success Probability (DSP) versus the number of samples $N$ with \texttt{RandomSwapPerturbation}, using the fitted attack success rate model.  We assume a total prompt length of $m = 240$ characters, an adversarial suffix length of $m_S = 100$, perturbation rate $q = 0.10$, threshold $k = 8$, and $\varepsilon = 0.05$. }
    \label{fig:comparison}
\end{figure}

\subsection{Instantiating the (k, $\varepsilon$)-Unstable framework}

Having established the theoretical framework and empirical motivation, we now demonstrate how to practically instantiate the (k, $\varepsilon$)-unstable assumption using our fitted data. From our empirical analysis in Figure 2, we observe that the attack success rate follows the fitted model $ASR(k) = ae^{-bk} + c$ with parameters $a = 0.1650$, $b = 0.1121$, and $c = 0.0427$ for RandomPatchPerturbation.

To instantiate our framework, consider a security requirement where we want certification against perturbations of $k = 10$ characters. From our empirical fit, $ASR(10) = 0.1650 \times e^{(-0.1121×10)} + 0.0427 \approx 0.097$. We can therefore set $\varepsilon = 0.10$ to instantiate our (k, $\varepsilon$)-unstable framework, meaning we assume that perturbations of 10 or more characters fail with probability at least 90\%.

This approach enables practitioners to select $k$ based on their security requirements, estimate $\varepsilon$ from validation data using the fitted ASR model, obtain realistic certificates that account for empirically observed attack robustness, and balance security guarantees with practical risk tolerance. The resulting certificate provides a probabilistic safety guarantee that is both mathematically rigorous and grounded in observed attack behavior, bridging the gap between theoretical assumptions and practical deployment needs.

\subsection{A practical case study: from risk tolerance to defense parameters}
\label{sec:case-study}

To demonstrate the practical utility of our framework, we present an end-to-end scenario that translates an organization's security requirements into concrete, actionable defense parameters for SmoothLLM.

\paragraph{1. Define the Goal} An organization needs to deploy the Llama2 model and requires a certified Defense Success Probability (DSP) of at least 95\% against GCG-style jailbreaking attacks.

\paragraph{2. Set the Risk Tolerance} The organization's security team determines that a residual attack success rate of $\varepsilon = 0.05$ is an acceptable risk for any prompt perturbation that modifies a sufficient number of characters. This means they require the attack to fail with at least 95\% probability once the perturbation threshold is met.

\paragraph{3. Determine the Perturbation Threshold ($k$)} Using empirical data from RandomSwapPerturbation on Llama2 (Figure 3) with fitted model $ASR(k) \approx 0.292e^{-0.376k} + 0.013$, we solve for the smallest $k$ such that $ASR(k) \leq 0.05$:
\begin{equation*}
0.292e^{-0.376k} + 0.013 \leq 0.05 \Rightarrow k \geq 5.49
\end{equation*}
Thus, the organization sets $k = 6$, providing the $(k, \varepsilon)$-unstable guarantee of $(6, 0.05)$. Notably, the more robust PAIR attack would require a higher $k$ for the same $\varepsilon = 0.05$, demonstrating how our framework adapts defense parameters to specific threat models.

\paragraph{4. Calculate the Number of Samples ($N$)} With $k=6$ and $\varepsilon=0.05$ established, we can now calculate the single-prompt defense probability, $\alpha$, using the tighter bound from Proposition 1. This requires computing $\alpha_{\text{tighter}}$ using the model parameters from Figure 1 ($m=240, m_S=100, q=0.10, M=24$). The probability $Pr[X=i]$ for a given number of suffix perturbations follows a hypergeometric distribution. By computing this value, we can then solve the DSP inequality for the minimum number of samples $N$:
\begin{equation*}
    \text{DSP} = \sum_{t=\lceil N/2 \rceil}^{N} \binom{N}{t} \alpha^t (1-\alpha)^{N-t} \geq 0.95
\end{equation*}
Solving this inequality reveals the minimum number of samples required to meet the 95\% defense guarantee (see Figure~\ref{fig:comparison}). For a computed $\alpha$ in this range, a value of $N=10$ is typically sufficient.

\paragraph{Impact and Actionable Insight} This case study provides a concrete, data-driven guide for practitioners. By starting with a high-level security goal (95\% DSP) and risk tolerance ($\varepsilon=0.05$), our framework allows them to derive the precise operational parameters ($k=6, N=10$) required for deploying SmoothLLM. This transforms the theoretical certificate into a practical tool for security engineering.

\paragraph{Guidance on Selecting $k$ and $\varepsilon$}
Our framework is designed for practical application, where practitioners derive certification parameters from their specific security posture. The selection process is a top-down, risk-driven approach: \textbf{1) Set Risk Tolerance ($\varepsilon$):} An organization first defines its maximum acceptable risk, such as $\varepsilon=0.05$ (a 5\% residual attack probability), based on its security policy. \textbf{2) Determine Perturbation Threshold ($k$):} With $\varepsilon$ set, $k$ is found empirically by measuring the Attack Success Rate (ASR) against the target LLM and attack. The minimum $k$ for which the measured $ASR(k) \leq \varepsilon$ becomes the certification threshold. Therefore, there is no universal "typical range"; the values are context-dependent. For a brittle attack like GCG on Llama2, $k \geq 6$ may suffice for $\varepsilon=0.05$, but a more robust attack would demand a higher $k$ to meet the same security guarantee.

\section{Conclusion and future work}
We addressed a fundamental limitation in SmoothLLM's certification: the gap between theoretical assumptions and empirical reality. Our primary contribution is a probabilistic certification framework based on (k,$\varepsilon$)-instability that provides more realistic and actionable security guarantees. By incorporating empirical attack success patterns through data-driven bounds, our approach yields certificates that are both mathematically rigorous and practically relevant for specific threat models.

Our approach has several limitations: the (k,$\varepsilon$)-unstable assumption requires empirical validation across diverse attack types and model architectures, the exponential decay parameters may not generalize beyond tested combinations, and our analysis was limited to 7B-parameter models. The framework assumes relatively independent perturbation effects, which may not hold for attacks exploiting specific linguistic structures, and $\varepsilon$ selection requires domain expertise. Future work should investigate theoretical foundations for exponential decay, adaptive $\varepsilon$ estimation methods, extension to semantic perturbations, and integration with other defense mechanisms.

The immediate impact is more trustworthy safety assessments for LLM deployments. Rather than relying on overly conservative worst-case assumptions, practitioners can now set evidence-based security thresholds that account for realistic attack behavior while maintaining formal guarantees. This transforms the defense from a theoretical construct into a flexible tool for protecting LLMs against exploitation, allowing organizations to manage deployment risk by balancing certified security guarantees, computational cost, and performance.

\bibliographystyle{unsrt}
\bibliography{bib}


\appendix

\section{Full proof of proposition 1}
\label{app:proof_prop1}

\begin{proof}
Following the SmoothLLM framework, the defense success probability is:
\begin{align}
    \text{DSP}(P) &= \Pr[\text{JB}(\text{SmoothLLM}(P)) = 0]\\
    &= \Pr\left[\frac{1}{N}\sum_{j=1}^N \text{JB}(\text{LLM}(P_j)) \leq \frac{1}{2}\right]
\end{align}
where $P_j \sim \mathbb{P}_q(P)$ are i.i.d. perturbed prompts. This equals the probability that at least $\lceil N/2 \rceil$ 
of the perturbed prompts do not jailbreak, which follows a binomial distribution:
\begin{align}
    \text{DSP}(P) = \sum_{t=\lceil N/2 \rceil}^{N} \binom{N}{t} \alpha^t (1-\alpha)^{N-t}
\end{align}
where $\alpha = \Pr_{Q \sim \mathbb{P}_q(P)}[\text{JB}(\text{LLM}(Q)) = 0]$.

To compute $\alpha$, let $X$ denote the random variable representing the number of characters in $S$ that are 
perturbed. Since we uniformly sample $M$ positions from $m$ total positions, $X$ follows a hypergeometric 
distribution:
\begin{align}
    \Pr[X = i] = \frac{\binom{m_S}{i}\binom{m-m_S}{M-i}}{\binom{m}{M}}, \quad i = 0, 1, \ldots, \min(M, m_S)
\end{align}

By the law of total probability:
\begin{align}
    \alpha = \sum_{i=0}^{\min(M,m_S)} \Pr[\text{JB}(\text{LLM}(Q)) = 0 \mid X = i] \cdot \Pr[X = i]
\end{align}

Under the $(k, \varepsilon)$-instability assumption:
\begin{itemize}
    \item For $i \geq k$: $\Pr[\text{JB}(\text{LLM}(Q)) = 0 \mid X = i] \geq 1 - \varepsilon$
    \item For $i < k$: No guarantee; let $\eta_i = \Pr[\text{JB}(\text{LLM}(Q)) = 0 \mid X = i]$
\end{itemize}

Therefore:
\begin{align}
    \alpha &= \sum_{i=0}^{k-1} \eta_i \cdot \Pr[X = i] + \sum_{i=k}^{\min(M,m_S)} \Pr[\text{JB}(\text{LLM}(Q)) = 0 \mid X = i] \cdot \Pr[X = i]\\
    &\geq \sum_{i=0}^{k-1} \eta_i \cdot \Pr[X = i] + \sum_{i=k}^{\min(M,m_S)} (1-\varepsilon) \cdot \Pr[X = i] \quad \text{(by $(k,\varepsilon)$-instability)}
\end{align}

For a conservative lower bound, we set $\eta_i = 0$ for all $i < k$ (worst-case assumption), yielding:
\begin{align}
    \alpha \geq \alpha_{\text{lower}} = (1-\varepsilon) \sum_{i=k}^{\min(M,m_S)} \frac{\binom{m_S}{i}\binom{m-m_S}{M-i}}{\binom{m}{M}}
\end{align}

However, the bound $\alpha_{\text{lower}}$ is overly conservative. The worst-case assumption that $\eta_i=0$ for $i < k$ ignores the high probability of defense success even for sub-threshold perturbations. In practice, $\eta_i$ is almost always greater than zero, as a perturbation of fewer than $k$ characters can still neutralize the jailbreak if it hits a critical character in the suffix or if a concurrent perturbation in the goal string $G$ disrupts the attack logic. To form a tighter, more realistic bound, we can leverage our empirical findings from Figures 1 and 2. The Attack Success Rate, $\text{ASR}(i)$, is well-approximated by an exponential decay model of the form $\text{ASR}(i) \approx a e^{-bi} + c$. Using the fitted parameters from our experiments ($a=0.2921$, $b=0.3756$, $c=0.0133$), we can define a model-informed estimate for $\eta_i = 1 - \text{ASR}(i)$. This yields a tighter lower bound, $\alpha_{\text{tighter}}$:
\begin{align}
    \alpha \geq \alpha_{\text{tighter}} = \sum_{i=0}^{k-1} \left(1 - (a e^{-bi} + c)\right) \cdot \Pr[X = i] + \sum_{i=k}^{\min(M,m_S)} (1-\varepsilon) \cdot \Pr[X = i]
\end{align}
This data-driven bound provides a more accurate assessment of the defense's effectiveness in practice.

Note that the exponential decay model $\text{ASR}(i) \approx a e^{-bi} + c$ is theoretically motivated by the diminishing marginal impact of additional character perturbations. As more characters are modified, the probability of disrupting critical attack patterns increases exponentially, while the constant c captures residual success probability from attack robustness. This functional form has been validated across attack scenarios in adversarial ML literature \cite{shafieinejadRobustnessBackdoorbasedWatermarking2019}.

We note that while this exponential decay provides a strong empirical fit for the tested scenarios, these fitted parameters (\textit{a}, \textit{b}, \textit{c}) are specific to the model-attack pairing. Therefore, a practical deployment of this certification framework would necessitate re-calibrating this model on a validation set relevant to the specific threat model.

This completes the proof.
\end{proof}

\section{Proof sketch of proposition 2}
\label{app:proof_prop2}

\begin{proof}
The overall Defense Success Probability (DSP) follows a binomial distribution, identical to the proof for \texttt{RandomSwapPerturbation} (Proposition 1). The main task is to derive a tight lower bound for $\alpha$, the single-prompt defense success probability under patch perturbations.

By the law of total probability, $\alpha = \sum_{i=0}^{M} \Pr[\text{Defense Success} \mid X=i] \cdot \Pr[X=i]$. Using our ($k, \varepsilon$)-unstable assumption and the empirically fitted model for sub-threshold perturbations, we arrive at the lower bound $\alpha_{\text{patch}}$.

The probability $\Pr[X=i]$ is determined by counting the number of valid starting positions for a patch of length $M$ out of a total of $m-M+1$ possibilities. The precise combinatorial calculation depends on edge cases related to the relative lengths of the goal ($m_G$), suffix ($m_S$), and patch ($M$). We present the high-level result here and provide a full, rigorous derivation in Appendix \ref{app:patch_combinatorics}.
\end{proof}

\section{Detailed sensitivity analysis}
\label{app:sensitivity}

\paragraph{Monotonic Relationship}
The certified DSP is a monotonically decreasing function of $\varepsilon$. This can be seen by examining its relationship with $\alpha$, the single-prompt defense success probability. From Proposition 1, the lower bound on $\alpha$ is given by:
\begin{equation}
    \alpha \geq \alpha_{\text{lower}} = (1-\varepsilon) \sum_{i=k}^{\min(M, m_s)} \frac{\binom{m_s}{i}\binom{m-m_s}{M-i}}{\binom{m}{M}}
\end{equation}
Let $P_{k+}$ denote the probability that a random perturbation modifies $k$ or more characters in the suffix $S$:
\begin{equation}
    P_{k+} = \sum_{i=k}^{\min(M, m_s)} Pr[X=i] = \sum_{i=k}^{\min(M, m_s)} \frac{\binom{m_s}{i}\binom{m-m_s}{M-i}}{\binom{m}{M}}
\end{equation}
The bound on $\alpha$ can be expressed more compactly for both the conservative and tighter, model-informed cases:
\begin{align}
    \alpha_{\text{lower}} &= (1 - \varepsilon) P_{k+} \\
    \alpha_{\text{tighter}} &= \sum_{i=0}^{k-1} \eta_i \cdot Pr[X=i] + (1 - \varepsilon) P_{k+}
\end{align}
In both formulations, $\alpha$ is a linear function of $\varepsilon$. The DSP is the cumulative distribution function of a binomial variable with success probability $\alpha$:
\begin{equation}
    \text{DSP}(\alpha) = \sum_{t=\lceil N/2 \rceil}^{N} \binom{N}{t} \alpha^t (1-\alpha)^{N-t}
\end{equation}
Since $\text{DSP}(\alpha)$ is monotonically increasing with $\alpha$ for $\alpha \in [0, 1]$, and $\alpha$ is monotonically decreasing with $\varepsilon$, it follows that the certified DSP is a monotonically decreasing function of $\varepsilon$.

\paragraph{Quantifying the Degradation}
To quantify the sensitivity, we can analyze the partial derivative of the $\alpha$ bound with respect to $\varepsilon$. For both $\alpha_{\text{lower}}$ and $\alpha_{\text{tighter}}$, the derivative is identical:
\begin{equation}
    \frac{\partial \alpha}{\partial \varepsilon} = -P_{k+}
\end{equation}
This result provides a clear interpretation: the rate at which the single-prompt defense guarantee degrades with an increase in $\varepsilon$ is precisely equal to the negative probability of the perturbation being large enough ($i \geq k$) to meet the instability condition.

Using the chain rule, we can find the sensitivity of the final DSP to $\varepsilon$:
\begin{equation}
    \frac{\partial \text{DSP}}{\partial \varepsilon} = \frac{\partial \text{DSP}}{\partial \alpha} \cdot \frac{\partial \alpha}{\partial \varepsilon} = -P_{k+} \cdot \frac{\partial}{\partial \alpha} \left( \sum_{t=\lceil N/2 \rceil}^{N} \binom{N}{t} \alpha^t (1-\alpha)^{N-t} \right)
\end{equation}
Since both $P_{k+}$ and the derivative of the binomial CDF are non-negative, the overall derivative $\frac{\partial \text{DSP}}{\partial \varepsilon}$ is non-positive, confirming that the certified guarantee weakens as $\varepsilon$ increases. For practitioners, this means that as empirical evidence suggests an attack is more robust (higher $\varepsilon$), the certified probability of successfully defending against it will decrease. The magnitude of this decrease is governed by $P_{k+}$, which is determined by the defense parameters ($q$, $m$, $m_s$), and the slope of the binomial CDF, which is determined by the number of samples $N$.

This direct, quantifiable relationship allows practitioners to model the trade-off between assumed attack robustness and the level of certified safety, enabling more informed, risk-based decisions for LLM deployment.

\section{Combinatorics for patch perturbation}
\label{app:patch_combinatorics}

Here, we provide a detailed derivation for $\Pr[X=i]$, the probability that a randomly placed patch of length $M$ overlaps with an adversarial suffix $S$ of length $m_S$ by exactly $i$ characters. The total prompt $P=[G;S]$ has length $m = m_G + m_S$. A patch is defined by its starting position, chosen uniformly from the $m-M+1$ possibilities.

The number of starting positions that result in an overlap of size $i$ depends on whether the patch is fully contained within the goal or suffix, or if it straddles a boundary. Following the analysis framework of Robey et al. \cite{robey2023smoothllm}, we analyze the cases:

\begin{itemize}
    \item \textbf{Case 1: Full Overlap ($i=M$)}. This requires the patch to be entirely within the suffix, which is only possible if $M \leq m_S$. The number of valid starting positions is $(m_S - M + 1)$.
    
    \item \textbf{Case 2: No Overlap ($i=0$)}. This requires the patch to be entirely within the goal, which is only possible if $M \leq m_G$. The number of valid starting positions is $(m_G - M + 1)$.
    
    \item \textbf{Case 3: Partial Overlap ($1 \leq i < M$)}. This occurs when the patch straddles either the beginning or the end of the suffix.
    \begin{itemize}
        \item A patch starting in $G$ and ending in $S$ yields exactly one starting position for each overlap size $i \in \{1, \dots, \min(M-1, m_G)\}$.
        \item A patch starting in $S$ and ending after $S$ yields exactly one starting position for each overlap size $i \in \{1, \dots, \min(M-1, m_S)\}$.
    \end{itemize}
    If both the goal and suffix are long enough (i.e., $m_G \geq M-1$ and $m_S \geq M-1$), there are exactly two starting positions for each partial overlap size $i \in \{1, \dots, M-1\}$.
\end{itemize}

The probability $\Pr[X=i]$ is found by dividing the number of valid starting positions for a given $i$ by the total number of positions, $m-M+1$. A complete formulation, accounting for all edge cases (e.g., $m_G < M-1$), follows the piecewise derivation in the original SmoothLLM appendix.

\section{Additional experimental validation}\label{sec:experiments}

\begin{figure}[h!]
    \centering
    \begin{minipage}[t]{0.48\linewidth}
        \centering
        \includegraphics[width=\linewidth]{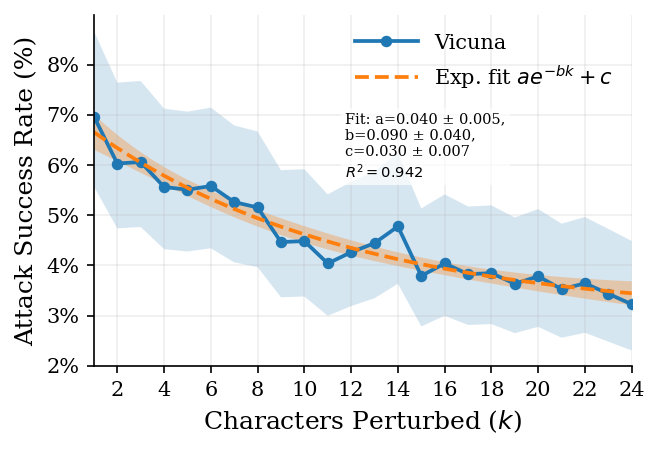}
        \caption{Attack success rate on \textbf{Vicuna} as a function of the number of perturbed characters $k$ using \texttt{RandomPatchPerturbation} and GCG attack.}
        \label{fig:vicunapatchgcg}
    \end{minipage}
    \hfill
    \begin{minipage}[t]{0.48\linewidth}
        \centering
        \includegraphics[width=\linewidth]{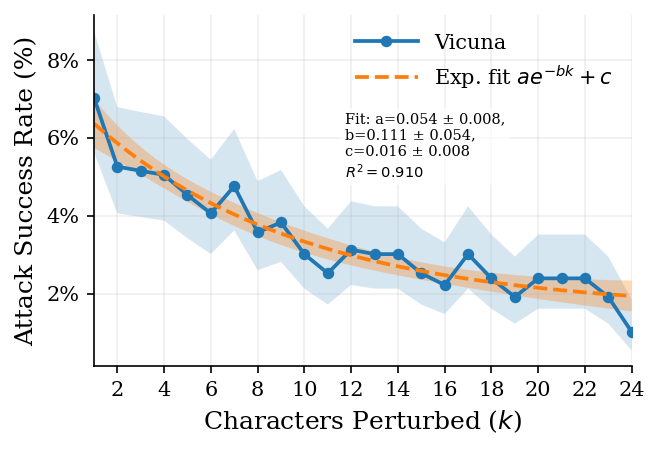}
        \caption{Attack success rate on \textbf{Vicuna} as a function of the number of perturbed characters $k$ using \texttt{RandomSwapPerturbation} and GCG attack.}
        \label{fig:vicunaswapgcg}
    \end{minipage}
\end{figure}

\begin{figure}[h!]
    \centering
    \begin{minipage}[t]{0.48\linewidth}
        \centering
        \includegraphics[width=\linewidth]{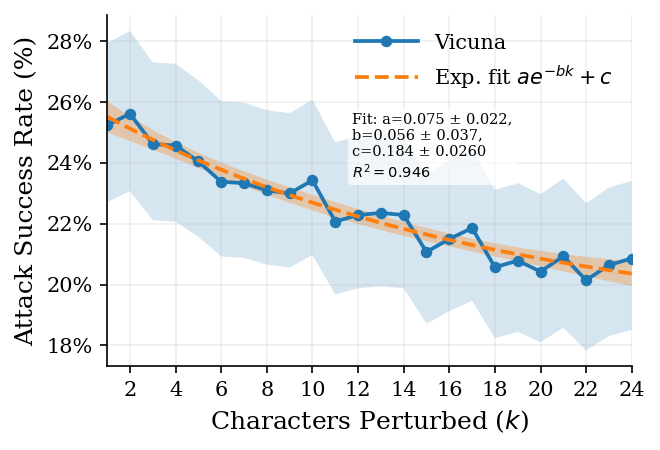}
        \caption{Attack success rate on \textbf{Vicuna} as a function of the number of perturbed characters $k$ using \texttt{RandomPatchPerturbation} and PAIR attack.}
        \label{fig:vicunapatchpair}
    \end{minipage}
    \hfill
    \begin{minipage}[t]{0.48\linewidth}
        \centering
        \includegraphics[width=\linewidth]{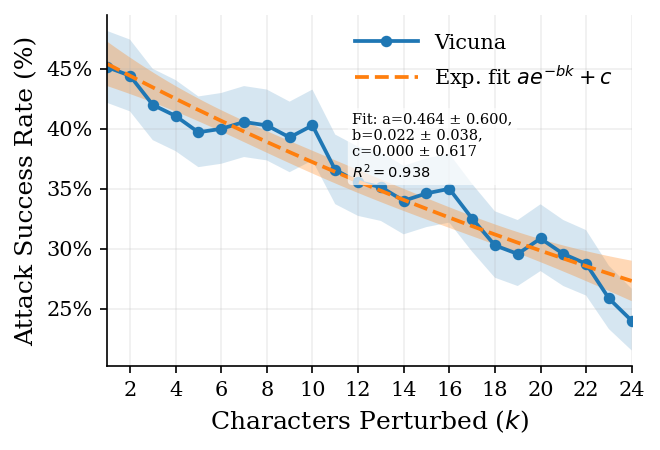}
        \caption{Attack success rate on \textbf{Vicuna} as a function of the number of perturbed characters $k$ using \texttt{RandomSwapPerturbation} and PAIR attack.}
        \label{fig:vicunaswappair}
    \end{minipage}
\end{figure}

\begin{figure}[h!]
    \centering
    \begin{minipage}[t]{0.48\linewidth}
        \centering
        \includegraphics[width=\linewidth]{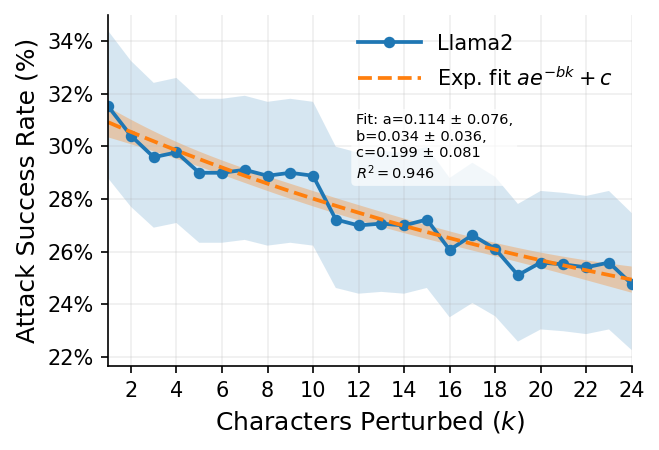}
        \caption{Attack success rate on \textbf{Llama2} as a function of the number of perturbed characters $k$ using \texttt{RandomPatchPerturbation} and PAIR attack.}
        \label{fig:llamapatchpair}
    \end{minipage}
    \hfill
    \begin{minipage}[t]{0.48\linewidth}
        \centering
        \includegraphics[width=\linewidth]{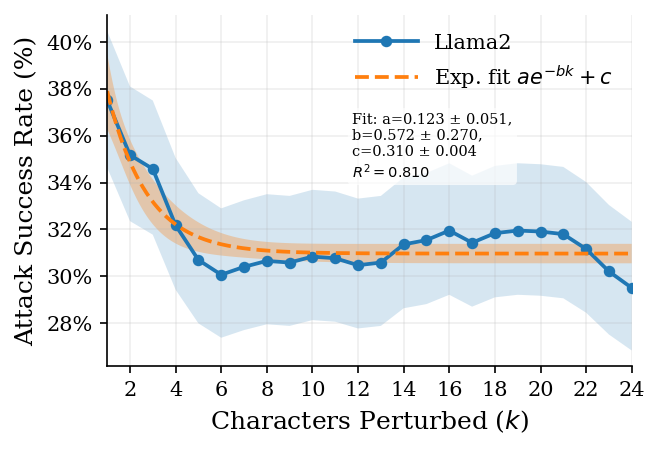}
        \caption{Attack success rate on \textbf{Llama2} as a function of the number of perturbed characters $k$ using \texttt{RandomSwapPerturbation} and PAIR attack.}
        \label{fig:llamaswappair}
    \end{minipage}
\end{figure}

\paragraph{Justification for Exponential Decay Model}
We model the relationship between the number of perturbed characters and the attack success rate (ASR) using:
\[
\mathrm{ASR}(i) \approx a\,e^{-b i} + c.
\]
This functional form is motivated by both empirical observations and prior work on adversarial suffix attacks. As shown in Figures~\ref{fig:vicunapatchgcg}--\ref{fig:llamaswappair}, ASR decreases rapidly when only a few characters are perturbed and then levels off, indicating that adversarial suffixes often rely on a small subset of \emph{critical} characters that drive the jailbreak logic. Perturbing these characters significantly disrupts the attack, while perturbing non-critical characters has little effect. This phenomenon is consistent with prior findings on GCG and PAIR jailbreaks, where specific trigger phrases or structural patterns are crucial to eliciting harmful outputs~\cite{zou2023universal, wei2023jailbroken}.

If we assume that each perturbation has an independent probability \(p\) of altering a critical character, the probability that none of the \(i\) perturbations disrupts the attack is approximately \((1-p)^i \approx e^{-b i}\), where \(b = -\ln(1-p)\). Across all tested models, attack strategies, and perturbation mechanisms, this exponential model fits the observed ASR curves with high accuracy (\(R^2 > 0.9\)). We therefore adopt \((a,b,c)\) as fitted parameters and compute \(\eta_i = 1 - \mathrm{ASR}(i)\) when deriving the tighter lower bound \(\alpha_{\mathrm{tighter}}\) used in our certification framework.

\paragraph{Variation Among Model/Attack Architecture}
While the exponential decay holds across settings, the parameters \((a,b,c)\) change depending on the model, the attack, and the perturbation strategy. We summarize the key trends below:

\begin{itemize}
    \item \textbf{Attack type (GCG vs.\ PAIR).} Our analysis reveals that the PAIR attack is quantifiably more robust to character-level perturbations than GCG, generally exhibiting a higher residual success rate ($c$) and a slower decay (smaller $b$). This empirically-grounded insight is critical for threat modeling, as it implies that defending against PAIR requires a higher perturbation threshold ($k$) or more samples ($N$) to achieve the same certified DSP. Our framework provides the necessary tool to calculate these differing requirements precisely.
    
    \item \textbf{Perturbation method (Swap vs.\ Patch).} \texttt{RandomSwapPerturbation} replaces individual characters independently, so it is more likely to hit critical positions early. This leads to larger \(b\) (faster decay). \texttt{RandomPatchPerturbation}, which perturbs a contiguous block, sometimes misses critical tokens or lands outside the suffix, producing a slower decay (smaller \(b\)) and a slightly larger floor \(c\).
    
    \item \textbf{Model architecture (Llama2 vs.\ Vicuna).} Vicuna generally shows lower initial attack success (\(a\)) and faster robustness gains with perturbations (larger \(b\)) compared to Llama2. However, the residual success \(c\) varies depending on how the jailbreak detector interacts with each model’s outputs.
\end{itemize}

These variations directly affect the certified defense success probability (DSP). A higher \(b\) and lower \(c\) mean that perturbations quickly neutralize attacks, leading to stronger certificates and fewer samples \(N\) required for a given DSP. Conversely, attacks with high \(a\) and \(c\) require either larger perturbation thresholds \(k\) or more samples to achieve the same safety guarantee. For this reason, we fit \((a,b,c)\) separately for each combination of model, attack, and perturbation method instead of assuming a single global model.

\paragraph{Deeper Analysis of Attack Fragility: GCG vs. PAIR}
Our empirical results reveal a stark contrast in the robustness patterns of GCG and PAIR attacks, which can be attributed to their fundamentally different mechanisms. The GCG attack, which uses a gradient-based search to append an often nonsensical adversarial suffix~\cite{zou2023universal}, demonstrates what can be termed \emph{syntactic fragility}. Its success relies on a precise sequence of characters that exploits specific, brittle vulnerabilities in the model's token processing. As shown in our experiments (e.g., Figures~\ref{fig:llamaswapgcg} and~\ref{fig:vicunaswapgcg}), this makes the attack highly susceptible to character-level perturbations; the ASR decays rapidly (a large fitted `b' parameter) and approaches a near-zero floor (a small `c'). A few random character changes are often sufficient to break the exact ``cheat code'' the attack relies on.

In contrast, the PAIR attack operates on a semantic level, using an LLM to iteratively refine a prompt to be more persuasive or to rephrase a harmful request in an innocuous-seeming way. This creates \emph{semantic resilience}. The jailbreak is triggered by the meaning and context of the prompt, not just a specific, fragile character string. Consequently, as seen in Figures~\ref{fig:vicunaswappair} and~\ref{fig:llamaswappair}, PAIR exhibits a much slower ASR decay (a smaller `b') and a significantly higher ASR floor (a large `c'). Changing a few characters is less likely to disrupt the core semantic instruction that bypasses the LLM's safety filters.

This distinction has crucial implications for defense and certification. It demonstrates that character-level defenses like SmoothLLM are exceptionally effective against syntactically fragile attacks but are less so against semantically resilient ones. More importantly, this validates the necessity of our (k, $\varepsilon$)-unstable framework. It allows practitioners to empirically quantify these different robustness profiles: a GCG-style threat might be certified with a small `k' and a very low $\varepsilon$, whereas a more resilient PAIR-style threat would require a much larger `k' or accepting a higher $\varepsilon$ to achieve the same certified DSP. Our framework thus provides a formal language to reason about and defend against the full spectrum of attack robustness.

\end{document}